\documentclass[12pt]{article}

\usepackage[english]{babel}

\usepackage[top=2.5cm,bottom=2.5cm,left=3cm,right=3cm,marginparwidth=1.75cm]{geometry}

\usepackage{amsmath}
\usepackage{amsthm}
\usepackage{mathtools}
\usepackage{amsfonts}
\usepackage{amssymb}
\usepackage{wasysym}
\usepackage{bbm}
\usepackage{mathbbol} 
\usepackage{xargs}
\usepackage{commath} 

\usepackage{graphicx}
\usepackage[colorlinks=true, allcolors=blue]{hyperref}
\usepackage[noabbrev,nameinlink,capitalise]{cleveref}
\usepackage{thmtools} 
\usepackage{cite}
\usepackage{natbib}

\usepackage{float}
\usepackage{calc}
\usepackage{tikz}
\usepackage{tikz-cd}
\usetikzlibrary{graphs,arrows,decorations.pathmorphing,decorations.markings,fit,positioning,hobby,arrows.meta}
\usepackage{subcaption}

\theoremstyle{definition}
\newtheorem{theorem}{Theorem}[section]

\newtheorem{corollary}[theorem]{Corollary}
\newtheorem{remark}[theorem]{Remark}

\newtheorem{proposition}[theorem]{Proposition}
\theoremstyle{definition}
\newtheorem{definition}[theorem]{Definition}

\newtheorem{example}[theorem]{Example}

\newcommand{\samethanks}[1][\value{footnote}]{\footnotemark[#1]}

\newcommand{\SO}{\mathrm{SO}}
\newcommand{\GL}{\mathrm{GL}}

\newcommand{\torsor}{\mathcal{P}} 
\newcommand{\Hom}{\mathrm{Hom}} 

\usepackage{longtable}

\usepackage{booktabs}
\usepackage[load-configurations=version-1]{siunitx} 
\usepackage{quiver}
\usepackage{tikz-cd}
\usepackage{enumitem}

 \title{Learning from Frustration: Torsor CNNs on Graphs}
\author{
  Daiyuan Li\thanks{Equal contribution, randomized order} \and
  Shreya Arya\samethanks \and
  Robert Ghrist\samethanks
}

\begin{document}

\maketitle

\begin{abstract}

Most equivariant neural networks rely on a single global symmetry, limiting their use in domains where symmetries are instead local. We introduce Torsor CNNs, a framework for learning on graphs with local symmetries encoded as edge potentials—group-valued transformations between neighboring coordinate frames. We establish that this geometric construction is fundamentally equivalent to the classical group synchronization problem, yielding: (1) a Torsor Convolutional Layer that is provably equivariant to local changes in coordinate frames, and (2) the frustration loss—a standalone geometric regularizer that encourages locally equivariant representations when added to any NN's training objective. The Torsor CNN framework unifies and generalizes several architectures—including classical CNNs and Gauge CNNs on manifolds—by operating on arbitrary graphs without requiring a global coordinate system or smooth manifold structure. We establish the mathematical foundations of this framework and demonstrate its applicability to multi-view 3D recognition, where relative camera poses naturally define the required edge potentials.
\end{abstract}

\section{Introduction}
\label{sec:intro}

Many learning problems involve data with local symmetries that vary across the domain. Consider a network of sensors measuring orientations in their own reference frames, a molecular structure where atomic neighborhoods exhibit different rotational symmetries, or a multi-view recognition, where each camera has its own coordinate system. These domains lack a global notion of orientation: each element knows only its local measurements and how its neighbors appear relative to its own frame. While existing equivariant architectures handle global symmetries (CNNs on grids, G-CNNs on homogeneous spaces) or require smooth manifold structure (Gauge CNNs), many real-world networks have arbitrary local transformations that don't fit these frameworks. See Section \ref{sec:works} for a detailed review.

\paragraph{Our Approach.} We introduce \emph{Torsor CNNs} for learning on graphs with local coordinate systems. The key input is an edge potential $\psi_{uv}$ for each edge $(u,v)$—this is a group element (like a rotation matrix) that tells us how to transform vectors from vertex $v$'s coordinate system to vertex $u$'s coordinate system. For example, if vertices are cameras, $\psi_{uv}\in \mathrm{SO}(3)$ is the rotation that aligns camera $v$'s view with camera $u$'s. These edge potentials let us build neural networks that respect the graph's geometry. Specifically, we construct layers that are gauge-equivariant: if someone changes all the local coordinate systems (rotating each camera differently, for instance), our network's output changes correspondingly. This means the network learns geometric relationships, not arbitrary coordinate choices.

 \paragraph{Connection to Synchronization.} The key insight is that this construction is mathematically equivalent to a classical problem called feature \emph{synchronization}. In this problem, we have features $f_v$ (vectors in some space $F$) at each vertex. The group $G$ acts on these features through a linear transformation $\rho$ -- for instance, if features are 3D vectors and $G = \text{SO}(3)$, then $\rho$ rotates vectors. Features are "synchronized" when they satisfy: $f_u = \rho(\psi_{uv}) f_v \quad \text{for every edge } (u,v)$.
This equation says that the feature at $u$ should equal the feature at $v$, transformed according to the edge potential. This connection immediately gives us a practical tool. We can measure how far any feature assignment is from being synchronized using the \emph{frustration loss}:
\[
L_{\text{frustration}}(f) = \sum_{(u,v) \in E} \|f_u - \rho(\psi_{uv})f_v\|^2
\]
 When this loss is zero, features perfectly respect the geometry. When it's large, features violate the geometric constraints. Remarkably, this loss can be added into diverse neural network objectives to encourage geometric consistency—you don't need specialized layers.

\paragraph{Generality.} Our framework unifies and extends existing methods.  Discretized versions of CNNs, G-CNNs, and Gauge CNNs all become special cases with particular choices of graphs and edge potentials. Unlike Gauge CNNs, which require edge potentials derived from parallel transport on a smooth manifold, our framework operates on any network-structured domain with arbitrary edge potentials, from measurements, domain knowledge, or learned from data.

\paragraph{Applications. } As a concrete example, we demonstrate that the framework can be applied to multi-view 3D object recognition, where graph nodes are camera views and edge potentials are their known relative $SO(3)$ rotations. See Appendix~\ref{sec:practical}. 

\paragraph{Contributions:} (1) We develop a discrete framework for local symmetries on graphs, using edge potentials to encode coordinate transformations between neighbors. (2) We prove that learning with local coordinate systems is equivalent to feature synchronization, connecting geometric deep learning to classical robotics or vision problems. (3) We provide two practical tools: Torsor Convolutional Layers that maintain geometric consistency by construction, and the frustration loss that encourages any neural network to learn geometrically consistent features. 

\section{Related Works: From Global to Local Equivariance}
\label{sec:works}

Equivariant deep learning has established itself as a powerful paradigm for improving data efficiency and generalization by incorporating the underlying geometric domain symmetries directly into neural architectures. In classical computer vision settings, such as planar images and spherical signals, global symmetry groups act transitively on the underlying geometric domain, e.g. 2D grids and 3D sphere. These global groups allow us to define convolutional operators that respect these invariances \cite{cohen2016group,cohen2019general}.  
Mathematically, these domains are modeled as \emph{homogeneous spaces} $M \cong G/H$, where $G$ is a Lie group acting transitively on $M$ and $H$ is a stabilizer subgroup. Feature fields are naturally represented as $H$-equivariant functions on $G$, satisfying the \emph{Mackey condition}, which can be interpreted in terms of sections of an associated vector bundle \cite{aronsson2022homogeneous, cohen2019general}. In this setting, $G$-equivariant linear maps correspond to convolutions with bi-equivariant kernels, recovering CNNs on $\mathbb{R}^2$ and spherical CNNs as special cases. However, many scientific and geometric learning problems lack such global transitive symmetries. On general manifolds, one must instead find local symmetries formalized by gauge theory \cite{cohen2019gauge,cohen2021equivariant}. Gauge-equivariant CNNs replace the global group action with a principal $G$-subbundle $P \subset FM$ with projection $\pi:P\to M$, where 
$FM=\bigsqcup_{p \in M} F_p=\{[v_1,\dots,v_d]\mid\{v_1,\dots,v_d\}\text{ is a basis of }T_pM\}$. Features are modeled as sections of associated vector bundles, and equivariance is guaranteed by defining convolutional operations using parallel transport induced by a connection on $P$. Yet, in discrete settings such as graphs, the absence of a smooth structure requires a new formulation. Inspired by the synchronization problem \cite{gao2021synchgeom} and the language of sheaves \cite{hansen2019toward, hansen2020sheaf, bodnar2021neural}, we adapt the recently introduced network torsors of \cite{ghrist2025obstructions}. 

\subsection{G-CNNs on Homogeneous Spaces}

A homogeneous space $M \cong G/H$ admits a transitive action by a Lie group $G$, with stabilizer $H$ at a base point. Feature fields on $M$ are represented as $H$-equivariant functions on $G$ satisfying the \textbf{Mackey condition}:
\[ 
f: G \to V \quad\text{with}\quad f(gh) = \rho(h)^{-1} f(g)\;\;\forall\;h\in H
\] 
where $\rho: H \to \GL(V)$ is a representation. This corresponds to sections of the associated bundle $E = G \times_H V$. Any $G$-equivariant linear map between such feature spaces can be written as a convolution with a kernel $\kappa: G \to \Hom(V, W)$ satisfying the bi-equivariance constraint:
\[ 
\kappa(h_2\, g\, h_1) = \rho_W(h_2)\, \kappa(g)\, \rho_V(h_1)^{-1}\quad \forall h_1,h_2 \in H
\]
This framework recovers classical CNNs on $\mathbb{R}^2$ and spherical CNNs \cite{cohen2018spherical} as special cases.

\subsection{Gauge Equivariant CNNs on Manifolds}

On general manifolds without global symmetry, local symmetries are formalized using a principal $G$-bundle $P \to M$ with structure group $G$ (e.g., $\SO(d)$ for oriented manifolds). Feature fields are sections of associated bundles $E = P \times_\rho F$. A gauge transformation $\gamma:U\to G$ acts on local representations as $f'(x) = \rho(\gamma(x))^{-1}f(x)$.

\noindent Gauge-equivariant convolutions use a connection to parallel-transport features between fibers. In local coordinates, the convolution takes the form:
\[
(\Phi f)(p) = \int_{\mathbb{R}^d} \kappa(v) \, \rho(g_{p \leftarrow q_v}) f(q_v) \, dv
\]
where $g_{p \leftarrow q_v} \in G$ represents parallel transport from $q_v$ to $p$. The kernel must satisfy $\kappa(g \cdot v) = \rho_{\text{out}}(g) \kappa(v) \rho_{\text{in}}(g)^{-1}$ for gauge-invariance \cite{cohen2019gauge,cohen2021equivariant}.

\noindent In discrete settings such as graphs, the absence of smooth structure requires new formulations. Inspired by synchronization \citep{singer2011angular,gao2021synchgeom} and sheaf theory \citep{hansen2019toward, hansen2020sheaf, bodnar2021neural}, we introduce network torsors that recover these constructions while enabling new capabilities for heterogeneous local symmetries.

\section{Mathematical Background}

Here we develop the discrete geometric structures underlying torsor CNNs. We begin with the group synchronization problem, then formalize local consistency via network sheaves, and define network $G$-torsors following \cite{ghrist2025obstructions} as discrete analogues of principal bundles.

\subsection{The Synchronization Problem}
\label{sec:sync}

Many problems in robotics \cite{rosen2019se}, structural biology \cite{singer2018mathematics}, and distributed sensing \cite{singer2011angular} involve recovering unknown global states from noisy relative measurements.

\begin{definition}
\label{def:group-sync}
Given a graph $X=(V,E)$ and edge measurements $\{\psi_{uv} \in G \mid \{u,v\} \in E\}$ satisfying $\psi_{vu} = \psi_{uv}^{-1}$, the \textbf{group synchronization problem} seeks a global assignment of states $\{g_v \in G\}_{v \in V}$ that best satisfies, for every edge $\{u,v\}\in E$,
\[
g_u = \psi_{uv}\,g_v.
\]
Here $\psi_{uv}$ is interpreted as the transformation mapping the state in frame $v$ to the state in frame $u$. A set of measurements is \textbf{consistent} if such an assignment exists and the relation holds with equality on all edges.
\end{definition}

\begin{example}[Planar Rotation Synchronization]
In the special case $G=\SO(2)$, each vertex represents an agent with an unknown orientation $g_v\in\SO(2)$ (equivalently, an angle $\theta_v\in(-\pi,\pi]$), and each edge measurement $\psi_{uv}\in\SO(2)$ encodes the relative rotation. This underlies sensor network calibration \cite{singer2011angular}, structure from motion \cite{eriksson18motion} and multi-view registration \cite{Arrigoni2016SpectralSE3}.
\end{example}

\noindent Consistency requires the product of transformations around any cycle to be the identity. In practice, noisy measurements violate this condition, and one seeks an assignment minimizing a global error objective known as \emph{frustration} \cite{singer2011angular}. This nonlinear problem can be linearized using a group representation:

\begin{definition}
\label{def:feature-sync}
Let $\rho: G \rightarrow \GL(F)$ be a linear representation. The \textbf{feature synchronization problem} seeks an assignment of feature vectors $\{f_v \in F\}_{v \in V}$ such that for every edge $\{u,v\} \in E$,
\[
f_u = \rho(\psi_{uv})\,f_v.
\]
\end{definition}

\subsection{Network Sheaves and Global Sections}
\label{sec:sheaves}

\begin{definition}
\label{def:network-sheaf}
A \textbf{network sheaf} $\mathcal{F}$ on a graph $X=(V,E)$ assigns a space $\mathcal{F}_v$ to each vertex $v\in V$ and a space $\mathcal{F}_e$ to each edge $e\in E$ (the \emph{stalks}). For each incidence of a vertex $v$ on an edge $e$, there is a morphism $\mathcal{F}_{v\to e}: \mathcal{F}_v \rightarrow \mathcal{F}_e$ (the \emph{restriction map}).
\end{definition}

\begin{definition}
A \textbf{global section} of a sheaf $\mathcal{F}$ is an assignment of elements $s_v \in \mathcal{F}_v$ to vertices such that for every edge $e=\{u,v\}$ the compatibility condition holds:
\[
\mathcal{F}_{u \to e}(s_u) = \mathcal{F}_{v\to e}(s_v).
\]
The set of all global sections is denoted $\Gamma(X,\mathcal{F})$.
\end{definition}

\subsection{Network Torsors from Edge Potentials}
\label{sec:torsors}

Our discrete formulation begins with \emph{edge potentials}, which serve as discrete connections from which we construct network torsors.
\begin{definition}
\label{def:edge_potential_first}
Given a graph $X=(V,E)$ and a group $G$, an \textbf{edge potential} is a function $\psi$ that assigns a group element $\psi_{uv}\in G$ to each oriented edge $e=(u,v)$, satisfying the antisymmetry property $\psi_{uv}=\psi_{vu}^{-1}$.
\end{definition}
\noindent The edge potential $\psi_{uv}$ maps a reference frame at vertex $v$ to the corresponding frame at vertex $u$. These local reference frames have no canonical origin -- they form a \emph{torsor}:
\begin{definition}
\label{def:g-torsor}
For a group $G$, a \textbf{$G$-torsor} is a nonempty set $P$ with a right action of $G$ that is free and transitive: for any $p,q\in P$ there exists a unique $g\in G$ with $q=p\cdot g$.
\end{definition}
\begin{definition}
\label{def:network-torsor-from-potential}
Let $\psi$ be an edge potential on a graph $X$ with group $G$. The \textbf{network $G$-torsor from $\psi$}, denoted $\torsor^{\psi}$, is the network sheaf defined as follows:
\begin{itemize}[noitemsep, topsep=0pt]
\item \textbf{Stalks:} $\torsor^{\psi}_v$ and $\torsor^{\psi}_e$ are the group $G$ itself, viewed as a $G$-torsor under right multiplication.
\item \textbf{Restriction maps:} For an oriented edge $e=(u,v)$,
\[
\torsor^{\psi}_{u \to e}(p)=p,\qquad 
\torsor^{\psi}_{v \to e}(p)=\psi_{uv}\,p,
\]
where juxtaposition denotes the group product in $G$.
\end{itemize}
\end{definition}

\noindent This construction yields a valid network $G$-torsor (see Appendix~\ref{app:network-torsor}). The compatibility condition for a global section $\{\sigma_v\}$ is exactly the group synchronization equation $\sigma_u=\psi_{uv}\,\sigma_v$. While we must choose an orientation $(u,v)$ to write these formulas, the antisymmetry property ensures orientation independence: on the reversed orientation one obtains $\sigma_v=\psi_{vu}\,\sigma_u$, which is equivalent since $\psi_{vu}=\psi_{uv}^{-1}$.

\subsection{Gauge and Gauge Transformations}
\label{sec:gauge}

To perform computations we coordinatize each torsor stalk using $1_G\in G$ as the origin—the \textbf{identity gauge}. Any other choice of reference frames is equally valid:

\begin{definition}
A \textbf{gauge transformation} is a map $\gamma:V\to G$, representing a change of reference frame at each vertex $v$ by the group element $\gamma_v$. Under $\gamma$, an edge potential $\psi$ transforms as
\[
\psi'_{uv}=\gamma_u^{-1}\,\psi_{uv}\,\gamma_v\qquad\text{for every oriented edge }(u,v).
\]
\end{definition}

\begin{definition}
Two edge potentials $\psi$ and $\psi'$ are \textbf{gauge-equivalent} if there exists a gauge transformation $\gamma$ relating them via the transformation law above.
\end{definition}

\section{Torsor CNNs on Graphs}
\label{sec:torsorCNNs}

We now use the network torsor structure to construct convolutional layers on graphs that are equivariant to local gauge transformations. The approach is a discrete analogue of gauge-equivariant CNNs on manifolds. While feature fields (global sections) represent the geometrically consistent features we ultimately seek, practical neural networks must handle arbitrary feature assignments that may not satisfy the strict synchronization condition. Therefore, our torsor convolution layer operates on the larger space $F^V$ of all feature assignments, while preserving the subspace of global sections when present.

\subsection{Associated Vector Sheaves and Feature Fields}

Given a graph $X=(V,E)$, a group $G$, and an edge potential $\psi :E\to G$, we have the induced network $G$-torsor (Definition~\ref{def:network-torsor-from-potential}). Together with a group representation $\rho: G \to \GL(F)$ (where $F$ is a feature vector space), we can construct an associated vector sheaf $\mathcal{E}$, analogous to the associated vector bundle in the continuous setting.

\begin{figure}
    \centering
    \includegraphics[width=\linewidth]{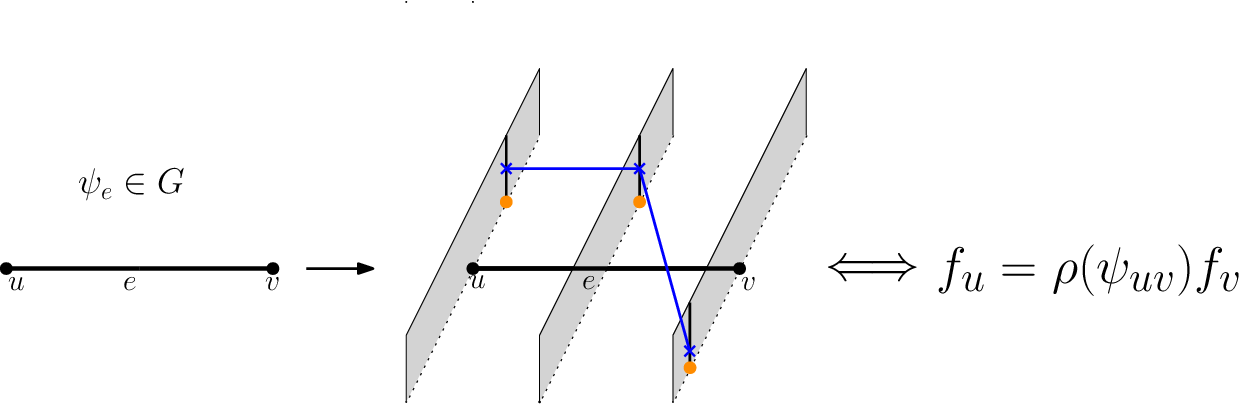}
    \caption{(Left) Edge $e = (u,v)$ with edge potential $\psi_e$ (or sometimes written as $\psi_{uv}$) mapping between local frames. (Right) Associated vector sheaf construction where each torsor stalk (dotted line) carries feature vector spaces at different reference points. Global section (blue) corresponds to synchronized features satisfying $f_u = \rho(\psi_{uv})f_v$ across the edge.}
    \label{fig:placeholder}
\end{figure}

\begin{definition}
\label{def:assoc-sheaf}
Let $\psi$ be an edge potential on $X=(V,E)$ with group $G$, and let $\rho: G \to \GL(F)$ be a finite-dimensional representation. The \textbf{associated vector sheaf} $\mathcal{E} = \torsor^\psi \times_\rho F$ is defined as follows:

\begin{itemize}
\item To each vertex and edge $x$, it assigns the vector space $\mathcal{E}_x := (\torsor^\psi_x \times F) / \sim$, where $(p\cdot g, w) \sim (p, \rho(g)w)$ for $p \in \torsor^\psi_x, g \in G, w \in F$.

\item The restriction maps $\mathcal{E}_{v\to e}: \mathcal{E}_v \rightarrow \mathcal{E}_e$ are induced by those of $\torsor^\psi$: $\mathcal{E}_{v\to e}([p,w]) := [\torsor^\psi_{v\to e}(p), w]$. 
\end{itemize}
\end{definition}

\noindent An element $[p,w] \in \mathcal{E}_v$ can be interpreted as a feature vector $w$ expressed in the frame $p$. The equivalence relation ensures that if we change the frame of reference by $g$, the coordinates of the vector transform by $\rho(g)$.

\begin{definition}
A \textbf{feature field} on the graph is a global section $\sigma$ of the associated vector sheaf $\mathcal{E}$. A section assigns an element $\sigma_v \in \mathcal{E}_v$ to each vertex $v$ such that for every edge $e=(u,v)$, the compatibility condition $\mathcal{E}_{u \to e}(s_u) = \mathcal{E}_{v\to e}(s_v)$ is satisfied.
\end{definition}

\noindent While a feature field $\sigma$ is an abstract object, computation requires a representation as a function $f: V \to F$. Since our edge potential $\psi$ corresponds to the identity gauge (where the identity $1_G \in G$ is the implicit reference frame at each vertex), we can represent any feature field as a function $f:V \to F$ solving the feature synchronization problem.

\begin{proposition}[Feature Fields as Solutions to Feature Synchronization]
\label{prop:sections-equals-sync}
 Given an edge potential $\psi$, there is a canonical bijection between feature fields $\Gamma(X, \mathcal{E})$ and functions $f:V \to F$ satisfying the feature synchronization condition:
\[ \Gamma(X, \mathcal{E}) \longleftrightarrow \{f: V \to F \mid f_u = \rho(\psi_{uv}) f_v \text{ for every edge } e=(u,v)\} \]
\noindent The bijection maps a section $\sigma$ to its representation $f$ in the identity gauge via $\sigma_v = [1_G, f_v]$, where $1_G \in G$ is the identity element. Moreover this correspondence is gauge-invariant: under a gauge transformation $\gamma: V \to G$ (transforming $\psi_{uv} \mapsto \psi_{uv}' :=\gamma_u^{-1}\psi_{uv}\gamma_v$), the same section is represented by $f'_v = \rho(\gamma_v)^{-1}f_v$, with the transformed features has the same form i.e. it satisfies the synchronization condition $f'_u = \rho(\psi_{uv}')f'_v$.

\end{proposition}

\noindent The proof is provided in Appendix~\ref{app:proof of prop_sections}. See also Figure~\ref{fig:placeholder} for a visual representation. The proposition establishes that a global section of the vector sheaf $\mathcal{E}$ is equivalent to a perfectly synchronized vector assignment $f$. However, an arbitrary function on the graph's vertices will not, in general, satisfy this strict condition. We can quantify the extent to which any given feature assignment $f: V \rightarrow F$ deviates from being a true feature field by measuring its total inconsistency across all edges. This is accomplished with a \emph{frustration loss} functional (see Appendix~\ref{app:proof of frustration} for a proof).

\begin{corollary}[The Frustration Functional]
\label{cor:frustration}
Assume $\rho$ is orthogonal and $f:V\to F$; define
\[
\eta_F(f;X,\psi)\;:=\;\frac{1}{\mathrm{vol}(X)}\sum_{\{u,v\}\in E}\bigl\|f_u-\rho(\psi_{uv})f_v\bigr\|^2,
\qquad \mathrm{vol}(X):=\sum_{v\in V}\deg(v)=2|E|.
\]
Then $\eta_F(f;X,\psi)=0$ if and only if $f$ represents a global section, and $\eta_F$ is gauge-invariant.
\end{corollary}

\noindent The gauge-invariance ensures we measure an intrinsic property of the feature field, not a coordinate artifact. In practice, the frustration loss serves as a regularization term: adding it to the baseline model's loss encourages learning feature representations consistent with the geometry of the underlying geometric domain of the data.

\subsection{Torsor Convolutional Layers}

A Torsor Convolutional Layer is a linear map $\Phi$ on feature assignments
$F^V$ that is \emph{gauge-equivariant} and preserves the subspace of global sections $\Gamma$.
If you transform the input feature field by a gauge transformation $\gamma$, the output of the layer is the transformed version of the original output. The layer works as follows: to compute the new feature at vertex $v$, we gather features from all neighboring vertices $u$, use the edge potentials $\psi$ to transport them into $v$'s local frame, apply a shared kernel, and then aggregate the results.

\begin{definition}
\label{def:torsor_conv_layer}
Let $\mathcal{E}_{\text{in}} = \torsor^\psi \times_{\rho_{\text{in}}} F_{\text{in}}$ and 
$\mathcal{E}_{\text{out}} = \torsor^\psi \times_{\rho_{\text{out}}} F_{\text{out}}$.
A \emph{Torsor Convolution Layer} is a gauge-equivariant linear map
$\Phi: F_{\text{in}}^{V} \longrightarrow F_{\text{out}}^{V}$
between feature assignments on vertices (not necessarily global sections),
parameterized by a learnable $G$-equivariant intertwiner $K: F_{\text{in}} \to F_{\text{out}}$ satisfying
\[
K(\rho_{\text{in}}(g)w)\;=\;\rho_{\text{out}}(g)\,K(w)\quad(\forall\,g\in G,\;w\in F_{\text{in}})
\]
(learned within the commutant of $\rho_{\text{in}}$ and $\rho_{\text{out}}$).
Given $f_{\text{in}}:V\to F_{\text{in}}$, the output $f_{\text{out}}=\Phi(f_{\text{in}})$ is
\[
f_{\text{out}}(v)\;=\;
\begin{cases}
\displaystyle \frac{1}{c_v}\sum_{u\sim v} w_{uv}\,
K\!\big(\rho_{\text{in}}(\psi_{uv})^{-1}\,f_{\text{in}}(u)\big),
& \text{if } c_v:=\sum_{u\sim v} w_{uv} > 0,\\[1.25em]
K\!\big(f_{\text{in}}(v)\big), & \text{if } c_v=0 \text{ (isolated $v$).}
\end{cases}
\]
Here $w_{uv}>0$ are optional edge weights (default $w_{uv}\equiv 1$).
The term $\rho_{\text{in}}(\psi_{uv})^{-1}f_{\text{in}}(u)$ transports the feature from $u$ into $v$'s local frame before applying $K$.
The normalization by $c_v$ ensures that, when restricted to global sections, the output is also a global section.
\end{definition}

\begin{remark}[Global Sections Preserved]
If $f_{\text{in}}\in\Gamma(X,\mathcal{E}_{\text{in}})$ satisfies
$f_{\text{in}}(u)=\rho_{\text{in}}(\psi_{uv})\,f_{\text{in}}(v)$ for all edges $\{u,v\}$,
then for every neighbor $u\sim v$,
\[
K\!\big(\rho_{\text{in}}(\psi_{uv})^{-1} f_{\text{in}}(u)\big)
=K\!\big(f_{\text{in}}(v)\big).
\]
Hence for $c_v>0$,
\[
f_{\text{out}}(v)=\frac{1}{c_v}\sum_{u\sim v} w_{uv}\,K(f_{\text{in}}(v)) = K(f_{\text{in}}(v)),
\]
and for isolated $v$ we set $f_{\text{out}}(v)=K(f_{\text{in}}(v))$ by definition.
Therefore $f_{\text{out}}(u)=\rho_{\text{out}}(\psi_{uv})\,f_{\text{out}}(v)$ for all edges, i.e. $f_{\text{out}}\in\Gamma(X,\mathcal{E}_{\text{out}})$.\footnote{A more expressive formulation replaces the single intertwiner $K$ by an edge-dependent kernel $\kappa(\psi_{uv})$ satisfying a bi-equivariance law.} 
\end{remark}

\begin{proposition}[Gauge Equivariance]\label{prop:gauge-eq-tor}
The torsor convolution layer is gauge-equivariant and preserves global sections. A proof is given in Appendix~\ref{app:proof of cnn layer}.
\end{proposition}

\begin{remark}[Equivariant nonlinearities]
While torsor convolutional layers provide the linear part of the architecture,
nonlinear activations must also preserve equivariance. In general, if features
transform according to a representation $\rho:G\to\GL(F)$, then any nonlinearity
$\sigma:F\to F$ must satisfy $\sigma(\rho(g)f)=\rho(g)\sigma(f)$ for all $g\in G$. 
For regular representations (where $G$
G acts on itself by permutation), pointwise nonlinearities like ReLU automatically satisfy this constraint.  However, for general irreducible representations—particularly of groups like $SO(n), SE(n)$—the equivariance constraint severely restricts allowed nonlinearities.  Valid constructions include norm-based activations that apply the nonlinearity only to the norm while preserving direction: $f \mapsto \sigma(\|f\|)\,\frac{f}{\|f\|}$
\citep{worrall2017harmonic, weiler2019general}, tensor product nonlinearities \citep{kondor2018nbody}, or gated nonlinearities where scalar fields modulate vector fields \citep{weiler2018steerable}. 
\end{remark}

\begin{remark}[Reduction to Known Architectures]
\label{rem:reduction}

\noindent\textbf{Classical CNNs on grids.}
    On a 2D grid viewed as a graph with $G=\mathbb{Z}^2$ (translations) and trivial representations, the torsor layer with a single intertwiner $K$ yields translation equivariance and location-wise weight sharing. It does \emph{not} reproduce position-selective filter taps of standard discrete convolution.%
    \footnote{Full recovery of classical position-selective filters can be obtained by replacing the single intertwiner $K$ with an edge-dependent kernel satisfying a bi-equivariance law; we omit this for space.}

\noindent\textbf{$G$-CNNs on homogeneous spaces.}
    For discretized homogeneous spaces $G/H$ with structure group $H$, the layer implements $H$-equivariant steering via $\rho(\psi_{uv})^{-1}$ and location-wise sharing. This captures the usual weight sharing across the domain; offset selectivity would again require the omitted edge-dependent kernel.

\noindent\textbf{Gauge-CNNs on manifolds.}
    On meshes with structure group $\SO(d)$, $\psi_{uv}\in\SO(d)$ encodes discrete parallel transport. The present layer provides gauge-equivariant weight sharing; richer dependence on relative orientations is possible with the omitted edge-dependent kernel. 
\end{remark}

\section{Discussion and Conclusion}
\label{sec:conc}

We introduced Torsor CNNs, a framework for learning on graphs with local symmetries encoded as edge potentials. The key insight—that gauge-equivariant learning and group synchronization are equivalent—yields both theoretical understanding and practical tools.
\paragraph{Practical Validation. } In Appendix~\ref{sec:practical}, we demonstrate the framework on ModelNet40 multi-view recognition. Camera poses provide natural $\text{SO}(3)$ edge potentials between views. We show two implementations: (A) Torsor CNN layers that explicitly transport features between camera frames before aggregation, and (B) standard multi-view networks (MVCNN, EMVN) augmented with frustration regularization. The frustration loss encourages view features to satisfy $f_i = \rho(\psi_{ij})f_j$ without architectural changes, reducing intra-class variance for improved retrieval mAP.

\paragraph{Future Work.} While we assumed a fixed group $G$ throughout, the framework naturally extends to heterogeneous settings, where different nodes could have different structure groups, for example, molecular graphs where single bonds allow $SO(2)$ rotations while double bonds have discrete $\mathbb{Z}_2$ symmetry. Another important direction is the development of standardized implementations: both torsor convolutional layers and frustration regularization should be distilled into practical and reusable modules.

 \paragraph{Conclusion.} Torsor CNNs provide a principled way to incorporate local geometric structure into graph learning. The frustration loss offers an immediate path to geometric regularization for any neural network, while the theoretical framework unifies CNNs, G-CNNs, and gauge CNNs as special cases of a general discrete theory. 


\bibliographystyle{plainnat}

\bibliography{reference}

\appendix

\section{Formalism of Network Torsors}
\label{app:network-torsor}

This appendix provides the formal definition of a network $G$-torsor and a verification that the construction from an edge potential (Definition~\ref{def:network-torsor-from-potential}) is a valid instance of this structure.

\begin{definition}
A \textbf{network $G$-torsor} $\torsor$ on a graph $X$ is a network sheaf satisfying two conditions:
\begin{enumerate}
\item The stalks $\torsor_v$ (for $v \in V$) and $\torsor_e$ (for $e \in E$) are all $G$-torsors.
\item The restriction maps $\torsor_{v \to e}: \torsor_v \to \torsor_e$ are $G$-equivariant. That is, for any $p \in \torsor_v$ and $g \in G$, the map respects the right group action:
\[ \torsor_{v\to e}(p \cdot g) = \torsor_{v\to e}(p) \cdot g \]
\end{enumerate}
\end{definition}

\begin{proposition}
    The network torsor from an edge potential from Definition~\ref{def:network-torsor-from-potential} is a network $G$-torsor. 
\end{proposition}
\begin{proof}
We now formally state and prove that the construction from Definition~\ref{def:network-torsor-from-potential} satisfies this definition.

The first condition requires that the stalks be $G$-torsors. By construction, every stalk in $\torsor^\psi$ is the group $G$ itself. A group $G$ forms a canonical $G$-torsor by acting on itself with right multiplication. This action is free and transitive, thus satisfying the definition of a $G$-torsor. 

Second, condition is verified by checking the $G$-equivariance of the restriction maps for an edge $e=(u,v)$. The map from the source vertex $u$, $\torsor^\psi_{u \to e}(p) = p$, is the identity and thus trivially equivariant. The map from the target vertex $v$, $\torsor_{v \to e}(p) = \psi_{uv} \cdot p$, is equivariant due to the associativity of the group operation:
\[ \torsor^\psi_{v \to e}(p \cdot g) = \psi_{uv} \cdot (p \cdot g) = (\psi_{uv} \cdot p) \cdot g = \torsor^\psi_{v \to e}(p) \cdot g \]
Since both conditions hold, the construction $\torsor^{\psi}$ yields a valid network $G$-torsor.
\end{proof}

\begin{definition}
A morphism $\phi:\mathcal P\to\mathcal Q$  is a collection of \textbf{$G$-equivariant maps} 
\[\{\phi_v:\mathcal P_v\to\mathcal Q_v\}_{v\in V} \quad \text{ and } \quad\{\phi_e:\mathcal P_e\to\mathcal Q_e\}_{e\in E}\] such that for every incidence $v\in e$ the diagram commutes:
\[\begin{tikzcd}
    {\mathcal{P}_v} & {\mathcal{Q}_v} \\
    {\mathcal{P}_e} & {\mathcal{Q}_e}
    \arrow["{\phi_v}", from=1-1, to=1-2]
    \arrow["{\mathcal{P}_{v \to e}}"', from=1-1, to=2-1]
    \arrow["{\mathcal{Q}_{v \to e}}", from=1-2, to=2-2]
    \arrow["{\phi_e}"', from=2-1, to=2-2]
\end{tikzcd}.\]
We call $\phi$ an isomorphism if each $\phi_v$ and $\phi_e$ is an isomorphism.
\end{definition}

\begin{proposition}
Let $\psi$ and $\psi'$ be gauge-equivalent edge potentials related by a gauge transformation $\gamma:V \to G$, such that $\psi'_{uv} = \gamma_u^{-1} \psi_{uv} \gamma_v$ for all edges $e=(u,v)$. Then the network torsors $\torsor^\psi$ and $\torsor^{\psi'}$ are isomorphic. 
\end{proposition}

\begin{proof}
We construct an explicit isomorphism of network sheaves, $\phi: \torsor^\psi \to \torsor^{\psi'}$. The isomorphism is defined on the stalks. For each vertex $v \in V$, we define the map $\phi_v:\torsor^\psi_v \to \torsor^{\psi'}_v$ by left multiplication:
\[ \phi_v(p) = \gamma_v^{-1} \cdot p \]
For each oriented edge $e=(u,v)$, we define the map on the edge stalk $\phi_e:\torsor^\psi_{e} \to \torsor^{\psi'}_e$ similarly, using the source vertex's transformation:
\[ \phi_e(p) = \gamma_u^{-1} \cdot p \]
To show that $\phi$ is a valid morphism of network sheaves, we must verify that the diagram of restriction maps commutes for every edge. In other words,
\[\begin{tikzcd}
    {\mathcal{P}^{\psi}_u} & {\mathcal{P}^{\psi'}_u} \\
    {\mathcal{P}^{\psi}_e} & {\mathcal{P}^{\psi'}_e}
    \arrow["{\phi_u}", from=1-1, to=1-2]
    \arrow["{\mathcal{P}^{\psi}_{u \to e}}"', from=1-1, to=2-1]
    \arrow["{\mathcal{P}^{\psi'}_{u \to e}}", from=1-2, to=2-2]
    \arrow["{\phi_e}"', from=2-1, to=2-2]
\end{tikzcd}\]
Consider an oriented edge $e=(u,v)$. For the source vertex $u$, the path through $\torsor^\psi_u \to \torsor^\psi_{e} \to \torsor^{\psi'}_e$ maps an element $p$ to $\phi_e(p) = \gamma_u^{-1} \cdot p$. The path through $\torsor^\psi_u \to \torsor^{\psi'}_u \to \torsor^{\psi'}_e$ maps $p$ to $\torsor_{u \to e}^{\psi'}(\phi_u(p)) = \torsor_{u \to e}^{\psi'}(\gamma_u^{-1} \cdot p) = \gamma_u^{-1} \cdot p$. The diagram commutes for the source vertex.\\
For the target vertex $v$, the path through $\torsor^\psi_v \to \torsor^\psi_{e} \to \torsor^{\psi'}_e$ maps an element $p$ to $\phi_e(\torsor_{v \to e}^\psi(p)) = \phi_e(\psi_{uv} \cdot p) = \gamma_u^{-1} \cdot (\psi_{uv} \cdot p)$. The path through $\torsor^\psi_v \to \torsor^{\psi'}_v \to \torsor^{\psi'}_e$ maps $p$ to $\torsor_{v \to e}^{\psi'}(\phi_v(p)) = \torsor_{v \to e}^{\psi'}(\gamma_v^{-1} \cdot p) = \psi_{uv}' \cdot (\gamma_v^{-1} \cdot p)$. Substituting the definition of $\psi_{uv}'$:
\[ \psi_{uv}' \cdot (\gamma_v^{-1} \cdot p) = (\gamma_u^{-1} \psi_{uv} \gamma_v) \cdot (\gamma_v^{-1} \cdot p) = \gamma_u^{-1} \psi_{uv} (\gamma_v \gamma_v^{-1}) p = \gamma_u^{-1} \cdot (\psi_{uv} \cdot p) \]
The diagram also commutes for the target vertex. Thus, $\phi$ is a morphism of network sheaves.
Finally, each map $\phi_v$ is a $G$-equivariant bijection. It is a bijection because left multiplication is invertible. It is $G$-equivariant because for any $g \in G$, $\phi_v(p \cdot g) = \gamma_v^{-1} \cdot (p \cdot g) = (\gamma_v^{-1} \cdot p) \cdot g = \phi_v(p) \cdot g$. The same holds for $\phi_e$. Therefore, $\phi$ is an isomorphism of network $G$-torsors.
\end{proof}

\section{Relegated Proofs}
\label{sec:proofs}

\subsection{Proof of Proposition~\ref{prop:sections-equals-sync}} \label{app:proof of prop_sections}

\begin{proof}
We work with the edge potential $\psi$ which corresponds to the identity gauge, where the identity element $1_G \in G$ serves as the implicit reference frame at each vertex.
Given a section $\sigma$ of the associated vector sheaf $\mathcal{E}$, we define a function $f:V \to F$ via $\sigma_v = [1_G, f_v]$. This gives a bijection between sections and synchronized functions.
For any edge $e=(u,v)$, the section compatibility condition $\mathcal{E}_{u \to e}(\sigma_u) = \mathcal{E}_{v \to e}(\sigma_v)$ holds if and only if:
\begin{align*}
\mathcal{E}_{u \to e}(\sigma_u) = \mathcal{E}_{v \to e}(\sigma_v) &\iff [\torsor^\psi_{u \to e}(1_G), f_u] = [\torsor^\psi_{v \to e}(1_G), f_v] \\
&\iff [1_G, f_u] = [\psi_{uv}, f_v] \\
&\iff [1_G, f_u] = [1_G, \rho(\psi_{uv})f_v] \\
&\iff f_u = \rho(\psi_{uv})f_v
\end{align*}
The first equivalence applies the definition of the restriction maps on $\mathcal{E}$. The second uses the restriction maps from Definition~\ref{def:network-torsor-from-potential}: $\torsor_{u \to e}^\psi(1_G) = e$ and $\torsor_{v \to e}^\psi(1_G) = \psi_{uv} \cdot 1_G = \psi_{uv}$. The third applies the equivalence relation $(p \cdot g, w) \sim (p, \rho(g)w)$ with $p=1_G$ and $g=\psi_{uv}$. The final step follows from the uniqueness of representation in the fiber $\mathcal{E}_e$. This establishes the bijection.

Consider a gauge transformation $\gamma:V \to G$ which transforms the edge potential to $\psi_{uv}' = \gamma_u^{-1} \psi_{uv} \gamma_v$. In the new gauge, the identity elements are replaced by $\gamma_v$ at each vertex $v$.
A section $\sigma$ that was represented by $f$ in the original (identity) gauge is now represented by $f'$ in the new gauge. Since $\sigma_v = [e, f_v] = [\gamma_v, f'_v]$ (the same element of $\mathcal{E}_v$ expressed in different coordinates), the equivalence relation gives us $f_v = \rho(\gamma_v)f'_v$, or equivalently, $f'_v = \rho(\gamma_v)^{-1}f_v$.
We verify that $f'$ satisfies the synchronization condition with respect to $\psi_{uv}'$:
\begin{align*}
f_u = \rho(\psi_{uv})f_v &\implies \rho(\gamma_u)f'_u = \rho(\psi_{uv})\rho(\gamma_v)f'_v \\
&\implies f'_u = \rho(\gamma_u)^{-1}\rho(\psi_{uv})\rho(\gamma_v)f'_v \\
&\implies f'_u = \rho(\gamma_u^{-1}\psi_{uv}\gamma_v)f'_v \\
&\implies f'_u = \rho(\psi_{uv}')f'_v
\end{align*}
Thus, the transformed function $f'$ satisfies the synchronization condition with the transformed edge potential $\psi_{uv}'$, confirming that the correspondence between sections and synchronized functions is gauge-independent. 
\end{proof}

\subsection{Proof of Corollary~\ref{cor:frustration}}\label{app:proof of frustration}

\begin{proof}
For consistency detection: each term $\|f_u-\rho(\psi_{uv})f_v\|^2$ is nonnegative, hence $\eta_F(f;X,\psi)=0$ iff all edge residuals vanish, i.e., $f_u=\rho(\psi_{uv})f_v$ for all $\{u,v\}\in E$, which is equivalent to $f$ representing a global section by Proposition~\ref{prop:sections-equals-sync}.
For gauge invariance: let $\gamma:V\to G$, and define $\psi'_{uv}=\gamma_u^{-1}\psi_{uv}\gamma_v$, $f'_v=\rho(\gamma_v)^{-1}f_v$. Then
\[
f'_u-\rho(\psi'_{uv})f'_v
=\rho(\gamma_u)^{-1}\bigl(f_u-\rho(\psi_{uv})f_v\bigr).
\]
Since $\rho$ is orthogonal, $\|\rho(\gamma_u)^{-1}w\|=\|w\|$ for all $w$, so each edge residual norm is unchanged. The sum and $\mathrm{vol}(X)$ are gauge-independent, hence $\eta_F(f';X,\psi')=\eta_F(f;X,\psi)$.
\end{proof}

\subsection{Proof of Proposition~\ref{prop:gauge-eq-tor}}\label{app:proof of cnn layer}

\begin{proof}
We must show that under a gauge transformation $\gamma:V \to G$, if the input transforms as $f'_{\text{in}}(v) = \rho_{\text{in}}(\gamma_v)^{-1}f_{\text{in}}(v)$ and the edge potential transforms as $\psi'_{uv} = \gamma_u^{-1}\psi_{uv}\gamma_v$, then $f'_{\text{out}}(v) = \rho_{\text{out}}(\gamma_v)^{-1}f_{\text{out}}(v)$.
The kernel $K:F_{\text{in}} \to F_{\text{out}}$ is $G$-equivariant:
\[
K(\rho_{\text{in}}(g)w) \;=\; \rho_{\text{out}}(g)\,K(w)\qquad(\forall\,g\in G,\;w\in F_{\text{in}}).
\]
If $c_v:=\sum_{u\sim v} w_{uv}>0$, then
\begin{align*}
f'_{\text{out}}(v)
&= \frac{1}{c_v}\sum_{u\sim v} w_{uv}\,K\!\left(\rho_{\text{in}}(\psi'_{uv})^{-1}\, f'_{\text{in}}(u)\right) \\
&= \frac{1}{c_v}\sum_{u\sim v} w_{uv}\,K\!\left(\rho_{\text{in}}(\gamma_v^{-1}\psi_{uv}^{-1}\gamma_u)\,\rho_{\text{in}}(\gamma_u)^{-1} f_{\text{in}}(u)\right) \\
&= \frac{1}{c_v}\sum_{u\sim v} w_{uv}\,K\!\left(\rho_{\text{in}}(\gamma_v^{-1})\,\rho_{\text{in}}(\psi_{uv}^{-1})\,f_{\text{in}}(u)\right) \\
&= \rho_{\text{out}}(\gamma_v)^{-1}\,\frac{1}{c_v}\sum_{u\sim v} w_{uv}\,K\!\left(\rho_{\text{in}}(\psi_{uv}^{-1})\,f_{\text{in}}(u)\right)
\;=\; \rho_{\text{out}}(\gamma_v)^{-1}\,f_{\text{out}}(v).
\end{align*}
If $c_v=0$ (isolated $v$), then $f'_{\text{out}}(v)=K(f'_{\text{in}}(v))=K(\rho_{\text{in}}(\gamma_v)^{-1}f_{\text{in}}(v))=\rho_{\text{out}}(\gamma_v)^{-1}K(f_{\text{in}}(v))=\rho_{\text{out}}(\gamma_v)^{-1}f_{\text{out}}(v)$ by the intertwining property. Thus the layer is gauge-equivariant in all cases.
\end{proof}

\section{Empirical Evaluation}\label{sec:practical}

To demonstrate the practical utility of our framework, we apply Torsor CNNs to multi-view 3D object recognition, where the geometric structure naturally aligns with our theoretical construction. In this setting, multiple cameras observe a 3D object from different viewpoints, and crucially, the relative orientations between cameras are known—providing exactly the edge potentials our framework requires. We show that this geometric information can be exploited in two complementary ways: (A) building fully gauge-equivariant Torsor CNN architectures that explicitly use the camera transformations in their convolutions, or (B) adding the frustration loss as a geometric regularizer to existing multi-view networks without architectural changes. 

\subsection {Rotated Multi-View 3D Recognition on ModelNet40}

\begin{enumerate}
    \item \textbf{Dataset:} $\mathcal{D}_{\text{mesh}}=\{(S_n,y_n)\}_{n=1}^M$ where $S_n$ is a CAD mesh and $y_n\in\mathcal{C}$ is the category label with $|\mathcal{C}|=40$ \cite{wu2015modelnet}.  Each mesh is rendered from $N$ views using a discrete camera set on $\mathbb{S}^2$.
   \item \textbf{Relative-Rotation Augmented Dataset: }  The relative rotation $\psi_{ij}\in SO(3)$ between views $i$ and $j$ is computed from known camera poses. Let $R_i,R_j\in SO(3)$ be the absolute rotation matrices of cameras $i$ and $j$, respectively. Then
    \[\psi_{ij}:=R_iR_j^\top\in SO(3),\qquad \psi_{ji}=\psi_{ij}^{-1}\]
    (since $R_j^{-1}=R_j^\top$). Define $\psi_{ij}$ as the \emph{edge potentials} on the view graph. Then each training input is augmented as
    \[x_n=\Big(\{I_{n,i}\}_{i=1}^N, \{\psi_{ij}\}_{(i,j)\in E}\Big) \]
     where $I_{n,i}$ is the $i$-th rendered view of the $n$-th object $S_n$. \\
     Then the underlying graph in the setup is the view-graph: its vertices correspond to different rendered views ${I_{n,i}}$ of the same 3D object, and its edges are annotated with relative rotations $\psi_{ij} \in SO(3)$ computed from camera poses

    \item \textbf{Tasks:} The learning tasks are 
    \begin{itemize}
        \item \emph{classification}: predict $y\in\mathcal{C}$
        \item \emph{retrieval}: find, for a query object $q$, the top-$K$ most similar objects $\{o_i\}_{i=1}^K$ from a dataset $\mathcal{D}$ by minimizing a distance metric $d(f(q), f(o_i))$, where $f$ denotes an embedding function. The model also produces a global descriptor used for ranking (evaluated by mAP), e.g. MVCNN \cite{su2015mvcnn} and equivariant multi-view networks (EMVN) \cite{esteves2019emvn} serving as baselines.
    \end{itemize}    
    
    \item \textbf{View-Graph as a $G$-Torsor:}
    
     Let $X=(V,E)$ be the view-graph and fix $G\subset SO(3)$. From known camera rotations $\{R_i\in SO(3)\}$ define the \emph{edge potential} $\psi_{ij}:= R_iR_j^\top\in G$. This induces a network $G$–torsor $\mathcal P_\psi$ (Definition \ref{def:network-torsor-from-potential}). Given a representation $\rho:G\to\mathrm{GL}(F)$, the associated vector sheaf is $\mathcal E=\mathcal P_\psi\times_\rho F$ (Definition \ref{def:assoc-sheaf}). In the identity gauge, sections $\sigma\in\Gamma(X,\mathcal E)$
     are represented by $f:V\to F$ satisfying the synchronization relation
     $f(i)=\rho(\psi_{ij})\,f(j)$ on every edge $(i,j)\in E$. A torsor–CNN layer acts by transporting neighbors to the local frame \[
     (\Phi f)(i)=\mathrm{activation} \Big(\sum_{j\sim i} w_{ij} K\big(\rho(\psi_{ij})\,f(j)\big)\Big)\]
     where the intertwiner $K:F \to F'$ satisfying $K \rho(g)=\rho'(g)\,K$ for all $g\in G$.
     For any gauge $\gamma:V\to G$, 
     \[ f^\gamma(i)=\rho(\gamma_i)^{-1}f(i), \;\;\psi_{ij}^\gamma=\gamma_i^{-1}\psi_{ij}\gamma_j\]
     one has $(\Phi f)^\gamma(i)=\rho'(\gamma_i)^{-1}(\Phi f)(i)$, i.e., the layer is gauge–equivariant.

     \item \textbf{Two Realizations of Torsor-Aware Learning:}
     \begin{enumerate}
         \item[A.] \textbf{Direct Enforcement via Torsor CNN.} We instantiate gauge–equivariant layers right above that explicitly transport features across views using the known edge potentials $\psi_{ij}$. Specifically speaking, for the $i$-th node
     \[(\Phi f)(i)= \mathrm{activation}\left(\sum_{j\sim i}w_{ij}K\big(\rho(\psi_{ij})f(j)\big)\right),\quad K\rho(g)=\rho'(g)K\]
     To obtain a global descriptor, we synchronize all node features to a fixed reference view $r$. Let $\psi_{ir}$ denote the potential along a path from $i$ to $r$; then the aligned feature is
     \[\hat f(i)=\rho(\psi_{ir})\,f(i)\]
     The global descriptor is the pooled representation
     \[z=\mathrm{Pool}_{i\in V}\,\hat f(i)\]
     where Pool can be mean/max pooling or attention. Each task is expected to inherit certain properties from this setting:
          \begin{itemize}
      \item \emph{Classification}: The synchronized global descriptor $z$ is passed to a classifier. Compared to baselines such as MVCNN that need to re-learn the geometry from the data, torsor CNNs embed the camera poses geometry directly. We thus expect improved accuracy with fewer views and better generalization under noisy or missing views.  
      \item \emph{Retrieval}: Using the synchronized descriptor $z$ as embedding, we train with a metric learning objective. Explicit synchronization reduces intra-class variance across camera poses, which is expected to yield higher mAP. For example, consider the standard triplet loss formulation \cite{schroff2015facenet,hermans2017defense}:
      \[\mathcal{L}_{\text{triplet}} = \max \left(0, \|f(x_i^a) - f(x_i^p)\|_2^2 - \|f(x_i^a) - f(x_i^n)\|_2^2 + \alpha \right)\] where $f(x_i^a)$ is the anchor feature (from a reference frame) for the $i$-th object, $f(x_i^p)$ is a positive feature from a different view of the same object and $f(x_i^n)$ is a negative feature from a different object. $\alpha > 0$ is a margin threshold.
     \end{itemize}
     Now, suppose we apply feature alignment using known camera pose transformations. Again, let $\psi_{ij}$ denote the edge potential (transformation) from view $j$ to view $i$, and $\rho$ be its representation in feature space (assumed to be isometric). We align all features to the anchor's view.\\
     Define the aligned features as $\hat{f}(x_i^a) = f(x_i^a)$ (anchor already in reference view), $\hat{f}(x_i^p) = \rho(\psi_{ap}) f(x_i^p) \text{ and } \hat{f}(x_i^n) = \rho(\psi_{an}) f(x_i^n)$. Assuming perfect alignment, the positive feature becomes identical to the anchor feature in the aligned space $\hat{f}(x_i^a) = \hat{f}(x_i^p)$. The triplet loss using aligned features is then:
     \begin{align*}
     \mathcal{L}_{\text{aligned}} 
     &= \max \left(0, \|\hat{f}(x_i^a) - \hat{f}(x_i^p)\|_2^2 - \|\hat{f}(x_i^a) - \hat{f}(x_i^n)\|_2^2 + \alpha \right)  \\
     &= \max \left(0, 0 - \|\hat{f}(x_i^a) - \hat{f}(x_i^n)\|_2^2 + \alpha \right)  \\
     &= \max \left(0, \alpha - \|\hat{f}(x_i^a) - \hat{f}(x_i^n)\|_2^2 \right)
     \end{align*}
     Thus, the loss depends only on the inter-class distance (between anchor and negative), as the intra-class distance (between anchor and positive) becomes zero.
     More generally, even if alignment is not perfect, it significantly reduces intra-class variance. Let $d_{\text{intra}} = \mathbb{E}[\|f(x_i^a) - f(x_i^p)\|2^2]$ be the expected intra-class distance without alignment, and $d_{\text{intra}}^{\text{aligned}} = \mathbb{E}[\|\hat{f}(x_i^a) - \hat{f}(x_i^p)\|2^2]$ with alignment. Effective alignment ensures $d_{\text{intra}}^{\text{aligned}} \ll d_{\text{intra}}$
     The goal of the triplet loss is to ensure that the intra-class distance is less than the inter-class distance by at least the margin $\alpha$, i.e., $d_{\text{intra}} < d_{\text{inter}} - \alpha$. After alignment, since $d_{\text{intra}}^{\text{aligned}}$ is greatly reduced, the inequality $d_{\text{intra}}^{\text{aligned}} < d_{\text{inter}}^{\text{aligned}} - \alpha$ is much easier to satisfy. This further makes the loss function easier to optimize. 
        \item[B.] \textbf{Frustration Energy as a Regularizer.}      Without altering the backbone (MVCNN, EMVN, etc.), we add a synchronization regularization term:
     \[\mathcal L_{\mathrm{sync}}
      =\sum_{(i,j)\in E}\|f(i)-\rho(\psi_{ij})f(j)\|^2\]
      The overall loss combines the task loss with $\lambda\mathcal L_{\mathrm{sync}}$.  Again, we expect that tasks can leverage certain structural properties:
      \begin{itemize}
          \item \emph{Classification}: The regularizer encourages alignment of view features consistent with the known geometry and thus the burden on pooling to discover it implicitly. We expect faster convergence with few or noisy views.
          \item \emph{Retrieval}: By penalizing frustration, embeddings for the same object under different viewpoints become more consistent. We thus expect higher mAP and better robustness for rare categories.
      \end{itemize}
     \end{enumerate}

\end{enumerate}

\end{document}